\documentclass[twoside,11pt]{article}

\usepackage{times}
\usepackage{fullpage}

\usepackage{graphicx} 
\usepackage{subfigure}

\usepackage[round]{natbib}

\usepackage{algorithm}
\usepackage{algorithmic}
\usepackage{hyperref}

\usepackage{multirow}
\usepackage{enumitem}
\usepackage{xspace}
\usepackage{color}
\usepackage{amssymb,amsfonts,latexsym,amsmath,amsthm}
\usepackage{xspace}

\theoremstyle{plain}

\newtheorem{lemma}{Lemma}

\usepackage{amsmath,amsfonts}
\usepackage{amsthm,amssymb,amsopn}
\usepackage{bm} 

\newcommand{\vct}[1]{\boldsymbol{#1}} 
\newcommand{\mat}[1]{\boldsymbol{#1}} 

\newcommand{\T}{^{\textrm T}} 

\DeclareMathOperator{\argmax}{arg\,max}
\DeclareMathOperator{\argmin}{arg\,min}

\newcommand{\eat}[1]{}

\newcommand{\innerp}[1]{\langle #1 \rangle}

\DeclareMathOperator{\diam}{diam}
\DeclareMathOperator{\diag}{diag}
\DeclareMathOperator*{\conv}{conv}

\newcommand{\mldiag}{\textsc{diag}\xspace}
\newcommand{\hdsl}{\textsc{hdsl}\xspace}
\newcommand{\ident}{\textsc{identity}\xspace}
\newcommand{\pcaml}{\textsc{pca+oasis}\xspace}

\newcommand{\rpml}{\textsc{rp+oasis}\xspace}

\newcommand{\svm}{\textsc{svm}\xspace}

\usepackage[title]{appendix}

\begin{document}

\title{Similarity Learning for High-Dimensional Sparse Data}

\author{
Kuan Liu\thanks{Equal contribution.}~~\thanks{Department of Computer Science, University of Southern California,
\texttt{\{kuanl,feisha\}@usc.edu}.}~, 
Aur\'elien Bellet\footnotemark[1]~~\thanks{LTCI UMR 5141, T\'el\'ecom ParisTech \& CNRS, \texttt{aurelien.bellet@telecom-paristech.fr}. Most of the work in this paper was carried out while the author was affiliated with Department of Computer Science, University of Southern California.}~, 
Fei Sha\footnotemark[2]
}

\date{}

\maketitle

\begin{abstract}
A good measure of similarity between data points is crucial to many tasks in machine learning. Similarity and metric learning methods learn such measures automatically from data, but they do not scale well respect to  the dimensionality of the data. In this paper, we propose a method that can learn efficiently similarity measure from high-dimensional sparse data. The core idea is to parameterize the similarity measure as a convex combination of rank-one matrices with specific sparsity structures. The parameters are then optimized with an approximate Frank-Wolfe procedure to maximally satisfy relative similarity constraints on the training data. Our algorithm greedily incorporates one pair of features at a time into the similarity measure, providing an efficient way to control the number of active features and thus reduce overfitting. It enjoys very appealing convergence guarantees and its time and memory complexity depends on the sparsity of the data instead of the dimension of the feature space. Our experiments on real-world high-dimensional datasets demonstrate its potential for classification, dimensionality reduction and data exploration.

\end{abstract}


\section{Introduction}

In many applications, such as text processing, computer vision or biology, data is represented as very high-dimensional but sparse vectors. The ability to compute meaningful similarity scores between these objects is crucial to many tasks, such as classification, clustering or ranking. However, handcrafting a relevant similarity measure for such data is challenging because it is usually the case that only a small, often unknown subset of features is actually relevant to the task at hand. For instance, in drug discovery, chemical compounds can be represented as sparse features describing their 3D properties, and only a few of them play an role in determining whether the compound will bind to a target receptor \citep{Guyon2004}. In text classification, where each document is represented as a sparse bag of words, only a small subset of the words is generally sufficient to discriminate among documents of different topics.

A principled way to obtain a similarity measure tailored to the problem of interest is to learn it from data. This line of research, known as similarity and distance metric learning, has been successfully applied to many application domains \citep[see][for recent surveys]{Kulis2012,Bellet2013}. The basic idea is to learn the parameters of a similarity (or distance) function such that it satisfies proximity-based constraints, requiring for instance that some data instance $\vct{x}$ be more similar to $\vct{y}$ than to $\vct{z}$ according to the learned function.
However, similarity learning typically requires estimating a matrix with $O(d^2)$ entries (where $d$ is the data dimension) to account for correlation between pairs of features. For high-dimensional data (say, $d > 10^4$), this is problematic for at least three reasons: (i) training the metric is computationally expensive (quadratic or cubic in $d$), (ii) the matrix may not even fit in memory, and (iii) learning so many parameters is likely to lead to severe overfitting, especially for sparse data where some features are rarely observed.

To overcome this difficulty, a common practice is to first project data into a low-dimensional space (using PCA or random projections), and then learn a similarity function in the reduced space. Note that the projection intertwines useful features and irrelevant/noisy ones. Moreover, it is also difficult to interpret the reduced feature space, when we are interested in discovering what features are more important than others for discrimination.

In this paper, we propose a novel method to learn a bilinear similarity function $S_{\mat{M}}(\vct{x},\vct{x}') = \vct{x}\T\mat{M}\vct{x}'$ directly in the original high-dimensional space while avoiding the above-mentioned pitfalls. The main idea combines three ingredients: the sparsity of the data, the parameterization of $\mat{M}$ as a convex combination of rank-one matrices with special sparsity structures, and an approximate Frank-Wolfe procedure \citep{Frank1956,Clarkson2010,Jaggi2013} to learn the similarity parameters. 
The resulting algorithm iteratively and greedily incorporates one pair of features at a time into the learned similarity, providing an efficient way to ignore irrelevant features as well as to guard against overfitting through early stopping. Our method has appealing approximation error guarantees, time and memory complexity independent of $d$ and outputs extremely sparse similarity functions that are fast to compute and  to interpret.

The usefulness of the proposed approach is evaluated on several datasets with up to 100,000 features, some of which have a large proportion of irrelevant features. To the best of our knowledge, this is the first time that a full similarity or distance metric is learned directly on such high-dimensional datasets without first reducing dimensionality. 
Our approach significantly outperforms both a diagonal similarity learned in the original space and a full similarity learned in a reduced space (after PCA or random projections). Furthermore, our similarity functions are extremely sparse (in the order of {0.0001\%} of nonzero entries), using a sparse subset of features and thus providing more economical analysis of the resulting model (for example, examining the importance of the original features and their pairwise interactions).

The rest of this paper is organized as follows. Section \ref{sRelate} briefly reviews some related work. Our approach is described in Section \ref{sApproach}. We present our experimental results in Section \ref{sExp} and conclude in Section \ref{sConclude}.


\section{Related Work}
\label{sRelate}

Learning similarity and distance metric has attracted a lot of interests. In this section, we review previous efforts that focus on efficient algorithms for  high-dimensional data -- a comprehensive survey of existing approaches can be found in \citep{Bellet2013}.

A majority of learning similarity has focused on learning either a Mahalanobis distance $d_{\mat{M}}(\vct{x},\vct{x}') = (\vct{x}-\vct{x}')\T\mat{M}(\vct{x}-\vct{x}')$ where $\mat{M}$ is a symmetric positive semi-definite (PSD) matrix, or a bilinear similarity $S_{\mat{M}}(\vct{x},\vct{x}') = \vct{x}\T\mat{M}\vct{x}'$ where $\mat{M}$ is an arbitrary $d\times d$ matrix. In both cases, it requires estimating $O(d^2)$ parameters, which is undesirable in the high-dimensional setting. Virtually all existing methods thus resort to dimensionality reduction (such as PCA or random projections) to preprocess the data when it has more than a few hundred dimensions, thereby incurring a potential loss of performance and interpretability of the resulting function \citep[see e.g.,][]{Davis2007,Weinberger2009,Guillaumin2009,Ying2012,Wang2012b,Lim2013,Qian2014}.

There have been a few solutions to this essential limitation. The most drastic strategy is to learn a diagonal matrix $\mat{M}$  \citep{Schultz2003,Gao2014}, which is very restrictive as it amounts to a simple weighting of the features.
Instead, some approaches assume an explicit low-rank decomposition $\mat{M}=\mat{L}\T\mat{L}$ and learn $\mat{L}\in\mathbb{R}^{r\times d}$ in order to reduce the number of parameters to learn \citep{Goldberger2004,Weinberger2009,Kedem2012}. But this results in nonconvex formulations with many bad local optima \citep{Kulis2012} and requires to tune $r$ carefully. Moreover, the training complexity still depends on $d$ and can thus remain quite large. Another direction is to learn $\mat{M}$ as a combination of rank-one matrices. In \citep{Shen2012}, the combining elements are selected greedily in a boosting manner but each iteration has an $O(d^2)$ complexity. To go around this limitation, \citet{Shi2014} generate a set of rank-one matrices before training and learn a sparse combination. However, as the dimension increases, a larger dictionary is needed and can be expensive to generate. Some work have also gone into sparse and/or low-rank regularization to reduce overfitting in high dimensions \citep{Rosales2006,Qi2009,Ying2009} but those do not reduce the training complexity of the algorithm.

To the best of our knowledge, DML-eig \citep{Ying2012} and its extension DML-$\rho$ \citep{Cao2012} are the only prior attempts to use a Frank-Wolfe procedure for metric or similarity learning. However, their formulation requires finding the largest eigenvector of the gradient matrix at each iteration, which scales in $O(d^2)$ and is thus unsuitable for the high-dimensional setting we consider in this work.


\section{Proposed Approach}
\label{sApproach}

This section introduces \hdsl (High-Dimensional Similarity Learning), the approach proposed in this paper. We first describe our problem formulation (Section~\ref{sec:form}), then derive an efficient algorithm to solve it (Section~\ref{sec:algo}).

\subsection{Problem Formulation}
\label{sec:form}

In this work, we propose to learn a similarity function for high-dimensional sparse data. We assume the data points lie in some space $\mathcal{X}\subseteq\mathbb{R}^d$, where $d$ is large ($d >10^4$), and are $D$-sparse on average ($D \ll d$). Namely, the number of nonzero entries is typically much smaller than $d$.
We focus on learning a similarity function $S_{\mat{M}}:\mathcal{X}\times\mathcal{X} \rightarrow \mathbb{R}$ of the form $S_{\mat{M}}(\vct{x},\vct{x}') = \vct{x}\T\mat{M}\vct{x}'$, where $\mat{M}\in\mathbb{R}^{d\times d}$. Note that for any $\mat{M}$, $S_{\mat{M}}$ can be computed in $O(D^2)$ time on average.

\paragraph{Feasible domain} Our goal is to derive an algorithm to learn a very sparse $\mat{M}$ with time and memory requirements that depend on $D$ but not on $d$. To this end, given a scale parameter $\lambda>0$, we will parameterize $\mat{M}$ as a convex combination of 4-sparse $d\times d$ bases:
\begin{equation*}
\mat{M}\in\mathcal{D}_\lambda = \conv(\mathcal{B}_\lambda),\quad\quad\text{ with } \mathcal{B}_\lambda = \bigcup_{ij} \left\{\mat{P}_\lambda^{(ij)}, \mat{N}_\lambda^{(ij)}\right\},
\end{equation*}
where for any pair of features $i,j\in\{1,\dots,d\}$, $i\neq j$,
{\arraycolsep=1.4pt\def\arraystretch{0.5}
\begin{equation*}
\mat{P}_\lambda^{(ij)} = \lambda(\boldsymbol{e}_i+\boldsymbol{e}_j)(\boldsymbol{e}_i+\boldsymbol{e}_j)^T = \left(\begin{array}{ccccc} \cdot & \cdot & \cdot & \cdot & \cdot\\ \cdot & \lambda & \cdot & \lambda & \cdot\\\cdot & \cdot & \cdot & \cdot & \cdot\\\cdot & \lambda & \cdot & \lambda & \cdot\\ \cdot & \cdot & \cdot & \cdot & \cdot\end{array}\right),
\mat{N}_\lambda^{(ij)} = \lambda(\boldsymbol{e}_i-\boldsymbol{e}_j)(\boldsymbol{e}_i-\boldsymbol{e}_j)^T = \left(\begin{array}{ccccc} \cdot & \cdot & \cdot & \cdot & \cdot\\ \cdot & \lambda & \cdot & -\lambda & \cdot\\\cdot & \cdot & \cdot & \cdot & \cdot\\\cdot & -\lambda & \cdot & \lambda & \cdot\\ \cdot & \cdot & \cdot & \cdot & \cdot\end{array}\right).
\end{equation*}
}The use of such sparse matrices was suggested by \citet{Jaggi2011a}. Besides the fact that they are instrumental to the efficiency of our algorithm (see Section~\ref{sec:algo}), we give some additional motivation for their use in the context of similarity learning.

First, any $\mat{M}\in\mathcal{D}_\lambda$ is a convex combination of symmetric PSD matrices and is thus also symmetric PSD. Unlike many metric learning algorithms, we thus avoid the $O(d^3)$ cost of projecting onto the PSD cone. Furthermore, constraining $\mat{M}$ to be symmetric PSD provides useful regularization to prevent overfitting \citep{Chechik2009} and allows the use of the square root of $\mat{M}$ to project the data into a new space where the dot product is equivalent to $S_{\mat{M}}$. Because the bases in $\mathcal{B}_\lambda$ are rank-one, the dimensionality of this projection space is at most the number of bases composing $\mat{M}$.

Second, each basis operates on two features only. In particular, $S_{\mat{P}_\lambda^{(ij)}}(\vct{x},\vct{x}') = \lambda(x_ix_i' + x_jx_j' + x_ix_j' + x_jx_i')$ assigns a higher score when feature $i$ appears jointly in $\vct{x}$ and $\vct{x}'$ (likewise for $j$), as well as when feature $i$ in $\vct{x}$ and feature $j$ in $\vct{y}$ (and vice versa) co-occur. Conversely, $S_{\mat{N}_\lambda^{(ij)}}$ penalizes the co-occurrence of features $i$ and $j$. This will allow us to easily control the number of active features and learn a very compact similarity representation.

Finally, notice that in the context of text data represented as bags-of-words (or other count data), the bases in $\mathcal{B}_\lambda$ are quite natural: they can be intuitively thought of as encoding the fact that a term $i$ or $j$ present in both documents makes them more similar, and that two terms $i$ and $j$ are associated with the same/different class or topic.

\paragraph{Optimization problem} We now describe the optimization problem to learn the similarity parameters. Following previous work \citep[see for instance][]{Schultz2003,Weinberger2009,Chechik2009}, our training data consist of side information in the form of triplet constraints:
$$\mathcal{T} = \left\{\vct{x}_t \text{ should be more similar to }\vct{y}_t \text{ than to } \vct{z}_t\right\}_{t=1}^T.$$
Such constraints can be built from a labeled training sample, provided directly by a domain expert, or obtained through implicit feedback such as clicks on search engine results. For notational convenience, we write $\mat{A}^t = \vct{x}_t(\vct{y}_t-\vct{z}_t)\T\in\mathbb{R}^{d\times d}$ for each constraint $t=1,\dots,T$. We want to define an objective function that applies a penalty when a constraint $t$ is not satisfied with margin at least 1, i.e. whenever $\innerp{\mat{A}^t,\mat{M}} = S_{\mat{M}}(\vct{x}_t,\vct{y}_t) - S_{\mat{M}}(\vct{x}_t,\vct{z}_t) < 1$. To this end, we use the smoothed hinge loss $\ell:\mathbb{R}\rightarrow\mathbb{R}^+$:
$$\ell\left(\innerp{\mat{A}^t,\mat{M}}\right) = \left\{ \begin{array}{ll} 0 & \text{if } \innerp{\mat{A}^t,\mat{M}} \geq 1\\\frac{1}{2} - \innerp{\mat{A}^t,\mat{M}} & \text{if } \innerp{\mat{A}^t,\mat{M}} \leq 0\\\frac{1}{2} \left(1-\innerp{\mat{A}^t,\mat{M}}\right)^2& \text{otherwise}\end{array}\right.,$$
where $\innerp{\cdot,\cdot}$ denotes the Frobenius inner product.\footnote{In principle, any other convex and continuously differentiable loss function can be used in our framework, such as the squared hinge loss, logistic loss or exponential loss.}

Given $\lambda>0$, our similarity learning formulation aims at finding the matrix $\mat{M}\in\mathcal{D}_\lambda$ that minimizes the average penalty over the triplet constraints in $\mathcal{T}$:

\begin{equation}
\label{eq:form}
\min_{\mat{M}\in\mathbb{R}^{d\times d}} \quad f(\mat{M}) = \frac{1}{T}\sum_{t=1}^T \ell\left(\innerp{\mat{A}^t,\mat{M}}\right)\quad\text{s.t.}\quad \mat{M}\in\mathcal{D}_\lambda.
\end{equation}

Due to the convexity of the smoothed hinge loss, Problem \eqref{eq:form} involves minimizing a convex function over the convex domain $\mathcal{D}_\lambda$. In the next section, we propose a greedy algorithm to solve this problem.

\subsection{Algorithm}
\label{sec:algo}

\subsubsection{Exact Frank-Wolfe Algorithm}

\begin{algorithm}[t]
\caption{Frank Wolfe algorithm for problem \eqref{eq:form}}
\label{alg:fw}
\begin{algorithmic}[1]
\STATE initialize $\mat{M}^{(0)}$ to an arbitrary $\mat{B}\in\mathcal{B}_\lambda$
\FOR{$k = 0,1,2,\dots$}
\STATE let $\mat{B}_{F}^{(k)} \in \argmin_{\mat{B}\in\mathcal{B}_\lambda} \innerp{\mat{B},\nabla f(\mat{M}^{(k)})}$ and $\mat{D}_{F}^{(k)} = \mat{B}_{F}^{(k)} - \mat{M}^{(k)}$ \hfill{\it\small // compute forward direction}
\STATE let $\mat{B}_{A}^{(k)} \in \argmax_{\mat{B}\in\mathcal{S}^{(k)}} \innerp{\mat{B},\nabla f(\mat{M}^{(k)})}$ and $\mat{D}_{A}^{(k)} =  \mat{M}^{(k)} - \mat{B}_{A}^{(k)}$ \hfill{\it\small // compute away direction}
\IF{$\innerp{\mat{D}_{F}^{(k)},\nabla f(\mat{M}^{(k)})} \leq \innerp{\mat{D}_{A}^{(k)},\nabla f(\mat{M}^{(k)})}$}
\STATE $\mat{D}^{(k)} = \mat{D}_{F}^{(k)}$ and $\gamma_{\max} = 1$\hfill{\it\small // choose forward step~~}
\ELSE
\STATE $\mat{D}^{(k)} = \mat{D}_{A}^{(k)}$ and $\gamma_{\max} = \alpha^{(k)}_{\mat{B}_{A}^{(k)}} / (1-\alpha^{(k)}_{\mat{B}_{A}^{(k)}})$\hfill{\it\small // choose away step~~}
\ENDIF
\STATE let $\gamma^{(k)} \in \argmin_{\gamma\in[0,\gamma_{\max}]} f(\mat{M}^{(k)}+\gamma \vct{D}^{(k)})$\hfill{\it\small // perform line search}
\STATE $\mat{M}^{(k+1)} = \mat{M}^{(k)} + \gamma^{(k)} \mat{D}^{(k)}$ \hfill{\it\small // update iterate towards direction~~}
\ENDFOR
\end{algorithmic}
\end{algorithm}

We propose to use a Frank-Wolfe (FW) algorithm \citep{Frank1956,Clarkson2010,Jaggi2013} to learn the similarity. FW is a general procedure to minimize a convex and continuously differentiable function over a compact and convex set. At each iteration, it moves towards a feasible point that minimizes a linearization of the objective function at the current iterate. Note that a minimizer of this linear function must be at a vertex of the feasible domain. We will exploit the fact that in our formulation \eqref{eq:form}, the vertices of the feasible domain $\mathcal{D}_\lambda$ are the elements of $\mathcal{B}_\lambda$ and have special structure.


The FW algorithm applied to \eqref{eq:form} and enhanced with so-called away steps \citep{Guelat1986} is described in details in Algorithm~\ref{alg:fw}. 
During the course of the algorithm, we explicitly maintain a representation of each iterate $\mat{M}^{(k)}$ as a convex combination of basis elements:
$$\mat{M}^{(k)} = \sum_{\mat{B}\in\mathcal{B}_\lambda}\alpha^{(k)}_{\mat{B}}\mat{B},\quad\quad \text{ where } \sum_{\mat{B}\in\mathcal{B}_\lambda}\alpha^{(k)}_{\mat{B}} = 1 \text{ and } \alpha^{(k)}_{\mat{B}}\geq 0.$$
We denote the set of active basis elements in $\mat{M}^{(k)}$ as $\mathcal{S}^{(k)} = \{\mat{B}\in\mathcal{B}_\lambda : \alpha^{(k)}_{\mat{B}} > 0\}$.
The algorithm goes as follows. We initialize $\mat{M}^{(0)}$ to a random basis element. Then, at each iteration, we greedily choose between moving towards a (possibly) new basis (forward step) or reducing the weight of an active one (away step). The extent of the step is determined by line search. As a result, Algorithm~\ref{alg:fw} adds only one basis (at most 2 new features) at each iteration, which provides a convenient way to control the number of active features and maintain a compact representation of $\mat{M}^{(k)}$ in $O(k)$ memory cost. Furthermore, away steps provide a way to reduce the importance of a potentially ``bad'' basis element added at an earlier iteration (or even remove it completely when $\gamma^{(k)} = \gamma_{\max}$). Note that throughout the execution of the algorithm, all iterates $\mat{M}^{(k)}$ remain convex combinations of basis elements and are thus feasible. The following lemma shows that the iterates of Algorithm~\ref{alg:fw} converge to an optimal solution of \eqref{eq:form} with a rate of $O(1/k)$.

%

\begin{lemma}
\label{lem:converge}
Let $\lambda>0$, $\mat{M}^*$ be an optimal solution to \eqref{eq:form} and $L = \frac{1}{T}\sum_{t=1}^T\|\mat{A}^t\|_F^2$. At any iteration $k\geq 1$ of Algorithm~\ref{alg:fw}, the iterate $\mat{M}^{(k)}\in\mathcal{D}_\lambda$ satisfies $f(\mat{M}^{(k)})-f(\mat{M}^*) \leq 16L\lambda^2/(k+2)$. Furthermore, it has at most rank $k+1$ with $4(k+1)$ nonzero entries, and uses at most $2(k+1)$ distinct features.
\end{lemma}
\begin{proof}
The result follows from the analysis of the general FW algorithm \citep{Jaggi2013}, the fact that $f$ has $L$-Lipschitz continuous gradient and observing that $\diam_{\|\cdot\|_F}(\mathcal{D}_\lambda)=\sqrt{8}\lambda$.
\end{proof}

Note that the optimality gap in Lemma~\ref{lem:converge} is independent from $d$. This means that Algorithm~\ref{alg:fw} is able to find a good approximate solution based on a small number of features, which is very appealing in the high-dimensional setting.

\subsubsection{Complexity Analysis}

We now analyze the time and memory complexity of Algorithm~\ref{alg:fw}. Observe that the gradient has the form:
\begin{equation}
\label{eq:gradient}
\nabla f(\mat{M}) = \frac{1}{T}\sum_{t=1}^T \mat{G}^t,\quad\quad\text{where }\mat{G}^t = \left\{ \begin{array}{ll} \mat{0} & \text{if } \innerp{\mat{A}^t,\mat{M}} \geq 1\\-\mat{A}^t&\text{if }\innerp{\mat{A}^t,\mat{M}} \leq 0\\\left(\innerp{\mat{A}^t,\mat{M}} - 1\right)\mat{A}^t& \text{otherwise}\end{array}\right..
\end{equation}
The structure of the algorithm's updates is crucial to its efficiency: since $\mat{M}^{(k+1)}$ is a convex combination of $\mat{M}^{(k)}$ and a 4-sparse matrix $\mat{B}^{(k)}$, we can efficiently compute most of the quantities of interest through careful book-keeping.

In particular, storing $\mat{M}^{(k)}$ at iteration $k$ requires $O(k)$ memory. We can also recursively compute $\innerp{\mat{A}^t,\mat{M}^{(k+1)}}$ for all constraints in only $O(T)$ time and $O(T)$ memory based on $\innerp{\mat{A}^t,\mat{M}^{(k)}}$ and $\innerp{\mat{A}^t,\mat{B}^{(k)}}$. This allows us, for instance, to efficiently compute the objective value as well identify the set of satisfied constraints (those with $\innerp{\mat{A}^t,\mat{M}^{(k)}} \geq 1$) and ignore them when computing the gradient.
Finding the away direction at iteration $k$ can be done in $O(Tk)$ time.
For the line search, we use a bisection algorithm to find a root of the gradient of the 1-dimensional function of $\gamma$, which only depends on $\innerp{\mat{A}^t,\mat{M}^{(k)}}$ and $\innerp{\mat{A}^t,\mat{B}^{(k)}}$, both of which are readily available. Its time complexity is $O(T\log \frac{1}{\epsilon})$ where $\epsilon$ is the precision of the line-search, and requires constant memory.

The bottleneck is to find the forward direction. 
Indeed, sequentially considering each basis element is intractable as it takes $O(Td^2)$ time. A more efficient strategy is to sequentially consider each constraint, which requires $O(TD^2)$ time and $O(TD^2)$ memory. The overall iteration complexity of Algorithm~\ref{alg:fw} is given in Table~\ref{tab:complexity}.

\subsubsection{Approximate Forward Step}
\label{sec:approach}

\begin{table}[t]
\centering
\begin{tabular}{|c||c|c|}
\hline Variant & Time complexity & Memory complexity \\
\hline\hline Exact (Algorithm~\ref{alg:fw}) & $\tilde{O}(TD^2+Tk)$ & $\tilde{O}(TD^2+k)$\\
\hline Mini-batch & $\tilde{O}(MD^2+Tk)$ & $\tilde{O}(T+MD^2 + k)$\\
\hline Mini-batch + heuristic & $\tilde{O}(MD+Tk)$ & $\tilde{O}(T+MD + k)$\\
\hline
\end{tabular}
\caption{Complexity of iteration $k$ (ignoring logarithmic factors) for different variants of the algorithm.}
\label{tab:complexity}
\end{table}

Finding the forward direction can be expensive when $T$ and $D$ are both large. We propose two strategies to alleviate this cost by finding an approximately optimal basis (see Table~\ref{tab:complexity} for iteration complexity).

\paragraph{Mini-Batch Approximation} Instead of finding the forward and away directions based on the full gradient at each iteration, we estimate it on a mini-batch of $M\ll T$ constraints drawn uniformly at random (without replacement). The complexity of finding the forward direction is thus reduced to $O(MD^2)$ time and $O(MD^2)$ memory. Under mild assumptions, concentration bounds such as Hoeffding's inequality without replacement \citep{Serfling1974} can be used to show that with high probability, the deviation between the ``utility'' of any basis element $\mat{B}$ on the full set of constraints and its estimation on the mini-batch, namely:
$$\left|\frac{1}{M}\sum_{t=1}^M\innerp{\mat{B},\mat{G}^t}-\frac{1}{T}\sum_{t=1}^T\innerp{\mat{B},\mat{G}^t}\right|,$$
decreases as $O(1/\sqrt{M})$. In other words, the mini-batch variant of Algorithm~\ref{alg:fw} finds a forward direction which is approximately optimal. The FW algorithm is known to be robust to this setting, and convergence guarantees similar to Lemma~\ref{lem:converge} can be obtained following \citep{Jaggi2013,Freund2013}. 

\paragraph{Fast Heuristic} To avoid the quadratic dependence on $D$, we propose to use the following heuristic to find a good forward basis. We first pick a feature $i\in\{1,\dots,d\}$ uniformly at random, and solve the linear problem over the restricted set $\bigcup_{j} \{\mat{P}_\lambda^{(ij)}, \mat{N}_\lambda^{(ij)}\}$. We then fix $j$ and solve the problem again over the set $\bigcup_{k} \{\mat{P}_\lambda^{(kj)}, \mat{N}_\lambda^{(kj)}\}$ and use the resulting basis for the forward direction. This can be done in only $O(MD)$ time and $O(MD)$ memory and gives good performance in practice, as we shall see in the next section.


%
%
%


\section{Experiments}
\label{sExp}

In this section, we present experiments to study the performance of HDSL in
classification, dimensionality reduction and data exploration against
competing approaches. Our Matlab code is publicly available on GitHub under
GNU/GPL 3 license.\footnote{\url{https://github.com/bellet/HDSL}} 

\subsection{Experimental Setup}

\paragraph{Datasets} We report experimental results on several real-world classification datasets with up to 100,000 features. Dorothea and dexter come from the NIPS 2003 feature selection challenge \citep{Guyon2004} and are respectively pharmaceutical and text data with predefined splitting into training, validation and test sets. They both contain a large proportion of noisy/irrelevant features. Reuters CV1 is a popular text classification dataset with bag-of-words representation. We use the binary version from the LIBSVM dataset collection\footnote{\url{http://www.csie.ntu.edu.tw/~cjlin/libsvmtools/datasets/}} (with 60\%/20\%/20\% random splits) and the 4-classes version (with 40\%/30\%/30\% random splits) introduced in \citep{Cai2012}. Detailed information on the datasets and splits is given in Table \ref{tDatasets}. All datasets are normalized such that each feature takes values in $[0,1]$.

\begin{table}[!t]
\centering
\begin{tabular}{|c||c|c|c|c|c|}
\hline
Datasets     & Dimension & Sparsity & Training size & Validation size & Test size \\ \hline\hline
dexter       & 20,000       & 0.48\%   & 300           & 300        & 2,000      \\ \hline
dorothea     & 100,000      & 0.91\%   & 800           & 350        & 800       \\ \hline
rcv1\_2 & 47,236    & 0.16\%   & 12,145         & 4,048       & 4,049      \\ \hline
rcv1\_4      & 29,992    & 0.26\%   & 3,850          & 2,888       & 2,887      \\ \hline
\end{tabular}
\caption{Datasets used in the experiments}
\label{tDatasets}
\end{table}


\paragraph{Competing Methods} We compare the proposed approach (\textsc{hdsl}) to several methods:
\begin{itemize}
\item \ident: The standard dot product as a baseline, which corresponds to using $\mat{M} = \mat{I}$. 
\item \mldiag: Diagonal similarity learning (i.e., a weighting of the features), as done in \cite{Gao2014}. We obtain it by minimizing the same loss as in \hdsl with $\ell_2$ and $\ell_1$ regularization, i.e.,
\begin{equation*}
\min_{\vct{w}\in\mathbb{R}^d} \quad f(\vct{w}) = \frac{1}{T}\sum_{t=1}^T \ell\left(\innerp{\mat{A}^t,\diag(\vct{w})}\right)+\lambda\Omega(\vct{w}),
\end{equation*}
where $\Omega(\vct{w}) \in \{\|\vct{w}\|_2^2,\|\vct{w}\|_1\}$ and $\lambda$ is the regularization parameter. Optimization is done using (proximal) gradient descent.
\item \rpml: Similarity learning in random projected space. Given $r\ll d$, let $\mat{R}\in\mathbb{R}^{d\times r}$ be a matrix where each entry $r_{ij}$ is randomly drawn from $\mathcal{N}(0,1)$. For each data instance $\vct{x}\in\mathbb{R}^d$, we generate $\tilde{\vct{x}} =\frac{1}{\sqrt{r}}\mat{R}\T\vct{x} \in \mathbb{R}^r$ and use this reduced data in OASIS \citep{Chechik2009}, a fast online method to learn a bilinear similarity from triplet constraints.
\item \pcaml: Similarity learning in PCA space. Same as \rpml, except that PCA is used instead of random projections to project the data into $\mathbb{R}^r$.
\item \svm: Support Vector Machines. We use linear SVM, which is known to perform well for sparse high-dimensional data \citep{Caruana2008}, with $\ell_2$ and $\ell_1$ regularization. We also use nonlinear SVM with the polynomial kernel (2nd and 3rd degree) popular in text classification \citep{Chang2010}. The SVM models are trained using liblinear \citep{Fan2008} and libsvm \citep{Chang2011} with 1vs1 paradigm for multiclass.
\end{itemize}

\paragraph{Training Procedure} For all similarity learning algorithms, we generate 15 training constraints for each instance by identifying its 3 target neighbors (nearest neighbors with same label) and 5 impostors (nearest neighbors with different label), following \cite{Weinberger2009}. Due to the very small number of training instances in dexter, we found that better performance is achieved using 20 constraints per instance $\vct{x}$, each of them constructed by randomly drawing a  point from the class of $\vct{x}$ and a point from a different class.
All parameters are tuned using the accuracy on the validation set. For \hdsl, we use the fast heuristic described in Section~\ref{sec:approach} and tune the scale parameter $\lambda \in \{1,10,\dots,10^9\}$. The regularization parameters of \mldiag and \svm are tuned in $\{10^{-9},\dots,10^8\}$ and the ``aggressiveness'' parameter of OASIS is tuned in $\{10^{-9},\dots,10^2\}$.

\subsection{Results} 

\begin{table}[t]
\centering
\begin{tabular}{|c||c|c|c|c|c|c|}
\hline
Datasets     & \ident  & \rpml & \pcaml &\mldiag-$\ell_2$ & \mldiag-$\ell_1$ & \hdsl \\ \hline\hline
dexter       & 20.1 & 24.0  {[}1000{]} & 9.3 [50] &8.4       & 8.4  {[}773{]}              & \textbf{6.5}  {[}183{]}         \\ \hline
dorothea     & 9.3    & 11.4 {[}150{]} & 9.9 [800]  & 6.8       & 6.6  {[}860{]}              & \textbf{6.5} {[}731{]}          \\ \hline
rcv1\_2 & 6.9  & 7.0  {[}2000{]}  & 4.5 [1500] &3.5       & 3.7  {[}5289{]}             & \textbf{3.4}  {[}2126{]}        \\ \hline
rcv1\_4      & 11.2 & 10.6  {[}1000{]} & 6.1 	[800] &6.2       & 7.2  {[}3878{]}             & \textbf{5.7}  {[}1888{]}        \\ \hline
\end{tabular}
\caption{$k$-NN test error (\%) of the similarities learned with each method. The number of features used by each similarity (when smaller than $d$) is given in brackets. Best accuracy on each dataset is shown in bold.}
\label{tKnnErr}
\end{table}

\begin{table}[t]
\centering
\begin{tabular}{|c||c|c|c|c|c|}
\hline
Datasets     & \svm-poly-2 & \svm-poly-3 & \svm-linear-$\ell_2$ & \svm-linear-$\ell_1$ & \hdsl \\ \hline\hline
dexter       & 9.4                   & 9.2                   & 8.9        & 8.9  {[}281{]}                     & \textbf{6.5}  {[}183{]}         \\ \hline
dorothea     & 7                     & 6.6                   & 8.1        & 6.6  {[}366{]}                    & \textbf{6.5} {[}731{]}          \\ \hline
	rcv1\_2 & 3.4                   & \textbf{3.3}                   & 3.5        & 4.0  {[}1915{]}                    & 3.4  {[}2126{]}        \\ \hline
rcv1\_4      & 5.7                   & 5.7                   & \textbf{5.1}        & 5.7  {[}2770{]}                    & 5.7 {[}1888{]}        \\ \hline
\end{tabular}
\caption{Test error (\%) of several SVM variants compared to \hdsl. As in Table~\ref{tKnnErr}, the number of features is given in brackets and best accuracies are shown in bold.}
\label{tCompareSVMs}
\end{table}

\paragraph{Classification Performance} We first investigate the performance of each similarity learning approach in $k$-NN classification ($k$ was set to 3 for all experiments). For \rpml and \pcaml, we choose the dimension $r$ of the reduced space based on the accuracy of the learned similarity on the validation set, limiting our search to $r\leq 2000$ because OASIS is extremely slow beyond this point.\footnote{Note that the number of PCA dimensions is at most the number of training examples. Therefore, for dexter and dorothea, $r$ is at most 300 and 800 respectively.} Similarly, we use the performance on validation data to do early stopping in \hdsl, which also has the effect of restricting the number of features used by the learned similarity.

Table~\ref{tKnnErr} shows the $k$-NN classification performance. First, notice that  \rpml often performs worse than \ident, which is consistent with previous observations that a large number of random projections may be needed to obtain good performance \citep{Fradkin2003}. \pcaml gives much better results, but is generally outperformed by a simple diagonal similarity learned directly in the original high-dimensional space. \hdsl, however, outperforms all other algorithms on these datasets, including \mldiag. This shows the good generalization performance of the proposed approach, even though the number of training samples is sometimes very small compared to the number of features, as in dexter and dorothea. It also shows the relevance of encoding ``second order'' information (pairwise interactions between the original features) in the similarity instead of simply considering a feature weighting as in \mldiag.


Table~\ref{tCompareSVMs} shows the comparison with SVMs. Interestingly, \hdsl with $k$-NN outperforms all SVM variants on dexter and dorothea, both of which have a large proportion of irrelevant features. This shows that its greedy strategy and early stopping mechanism achieves better feature selection and generalization than the $\ell_1$ version of linear SVM. On the other two datasets, \hdsl is competitive with SVM, although it is outperformed slightly by one variant (\svm-poly-3 on rcv1\_2 and \svm-linear-$\ell_2$ on rcv1\_4), both of which rely on all features.

\paragraph{Feature Selection and Sparsity} 

\begin{figure}[t]
\centering
\subfigure[dexter dataset]{\includegraphics[width=0.35\textwidth]{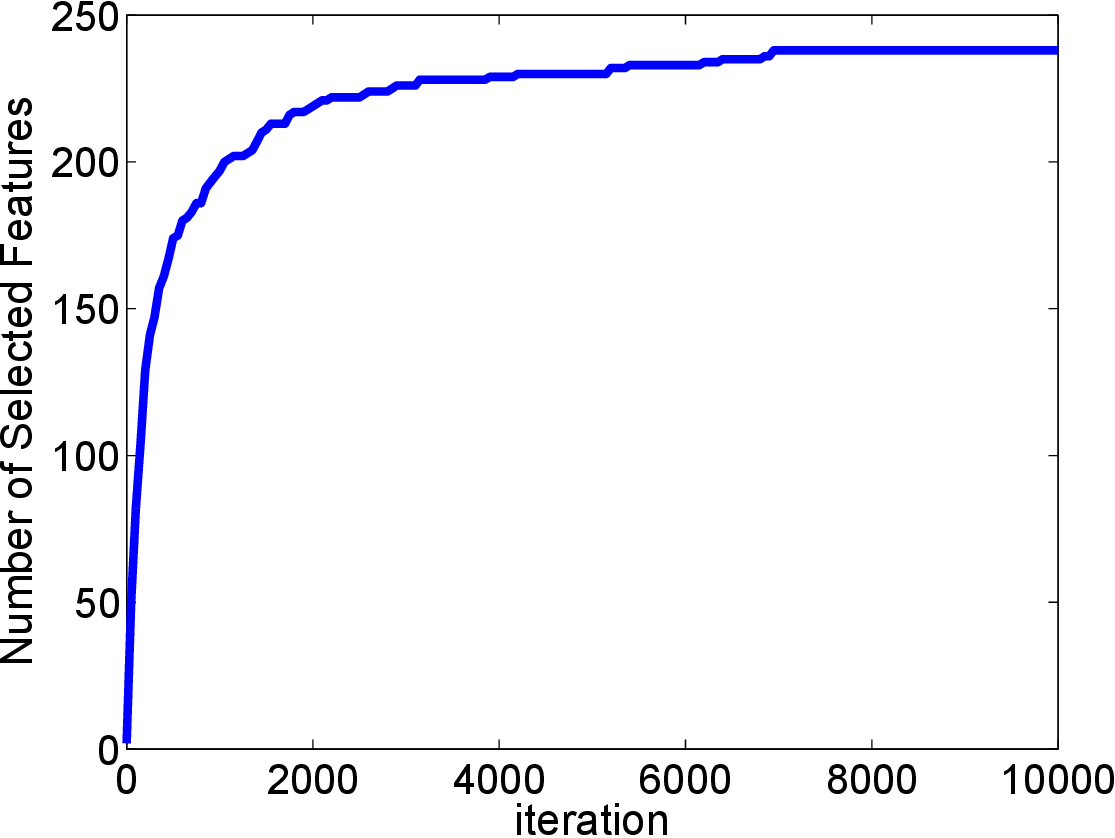}
\label{fNumFea1}}
\subfigure[rcv1\_4 dataset]{\includegraphics[width=0.35\textwidth]{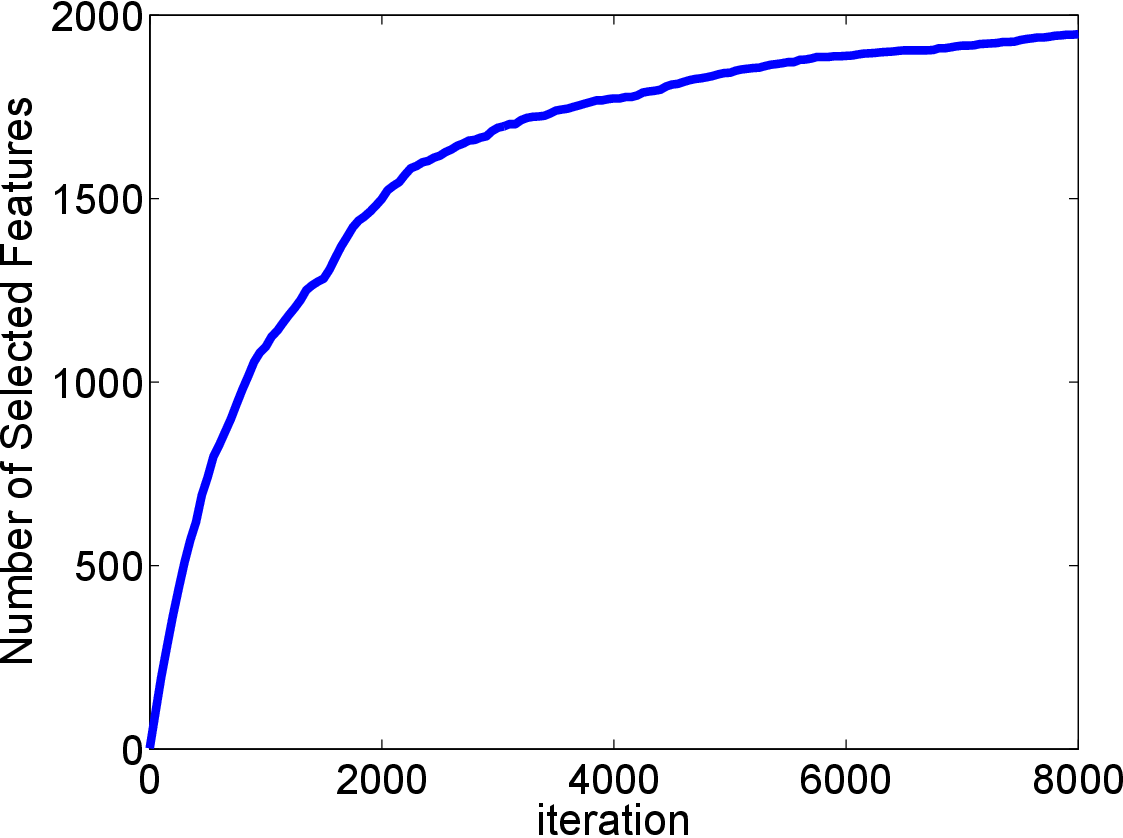} 
\label{fNumFea2}}
 \caption{Number of active features learned by \hdsl as a function of the iteration number.}\label{fSparseFeature}
 \label{fFeaNum}
\end{figure}

\begin{figure}[t]
\centering
\subfigure[dexter ($20,000\times20,000$ matrix, 712 nonzeros)]{\includegraphics[width=0.4\textwidth]{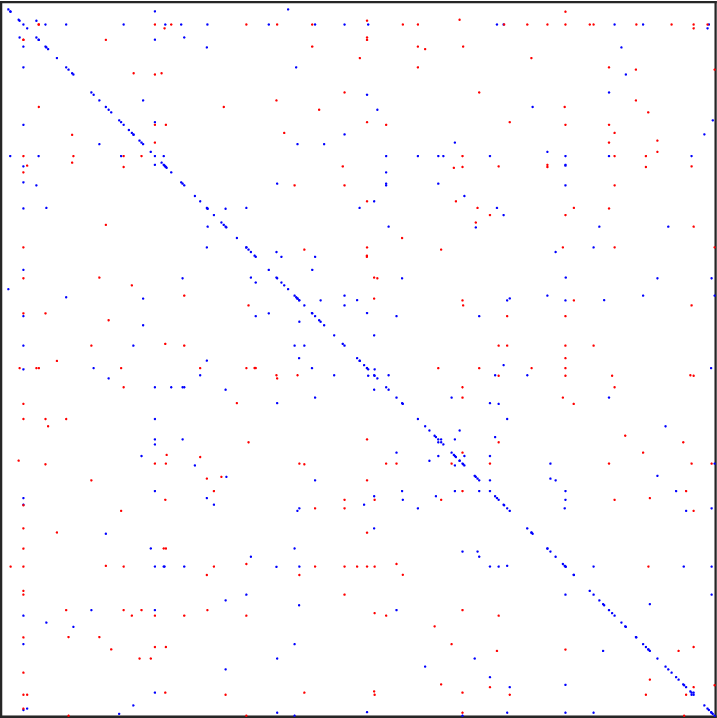}
\label{dexter_mat}}
\hspace*{1cm}\subfigure[rcv1\_4 ($29,992\times29,992$ matrix, 5263 nonzeros)]{\includegraphics[width=0.4\textwidth]{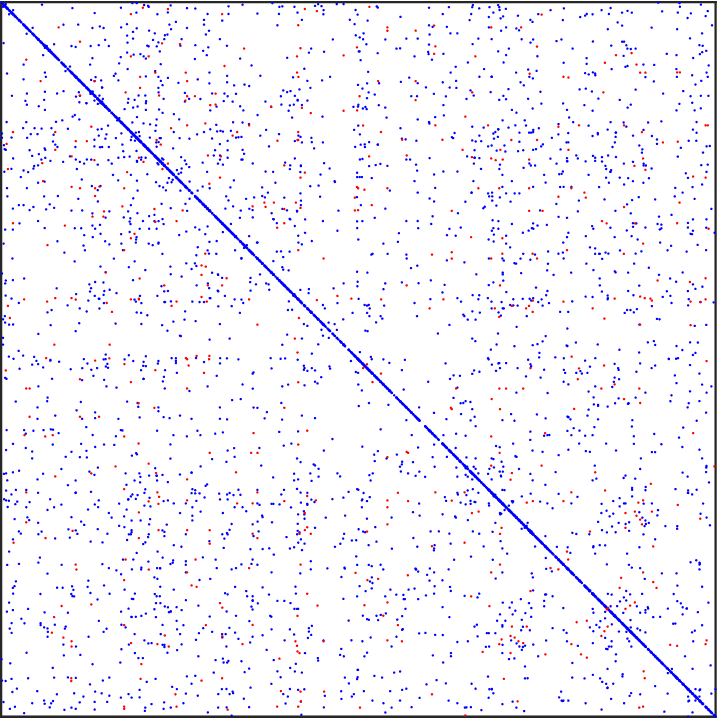} 
\label{rcv14_mat}}
 \caption{Sparsity structure of the matrix $\mat{M}$ learned by \hdsl. Positive and negative entries are shown in blue and red, respectively (best seen in color).}\label{}
 \label{hdsl_mat}
\end{figure}

We now focus on the ability of \hdsl to perform feature selection and more generally to learn sparse similarity functions. To better understand the behavior of \hdsl, we show in Figure~\ref{fFeaNum} the number of selected features as a function of the iteration number for two of the datasets. Remember that at most two new features can be added at each iteration. Figure~\ref{fFeaNum} shows that \hdsl incorporates many features early on but tends to eventually converge to a modest fraction of features (the same observation holds for the other two datasets). This may explain why \hdsl does not suffer much from overfitting even when training data is scarce as in dexter.

Another attractive characteristic of \hdsl is its ability to learn a matrix that is sparse not only on the diagonal but also off-diagonal (the proportion of nonzero entries is in the order of 0.0001\% for all datasets). In other words, it only relies on a few relevant pairwise interactions between features. Figure~\ref{hdsl_mat} shows two examples, where we can see that \hdsl is able to exploit the product of two features as either a positive or negative contribution to the similarity score. This opens the door to an analysis of the importance of pairs of features (for instance, word co-occurrence) for the application at hand. Finally, the extreme sparsity of the matrices allows very fast similarity computation. 

Finally, it is also worth noticing that \hdsl uses significantly less features than \mldiag-$\ell_1$ (see numbers in brackets in Table~\ref{tKnnErr}). We attribute this to the extra modeling capability brought by the non-diagonal similarity observed in Figure~\ref{hdsl_mat}.\footnote{Note that \hdsl uses roughly the same number of features as \svm-linear-$\ell_1$ (Table~\ref{tCompareSVMs}), but drawing any conclusion is harder because the objective and training data for each method are different. Moreover, 1-vs-1 SVM combines several binary models to deal with the multiclass setting.}





\paragraph{Dimension Reduction}

\begin{figure}[t]
\centering
\subfigure[dexter dataset]{
\includegraphics[width=0.41\textwidth]{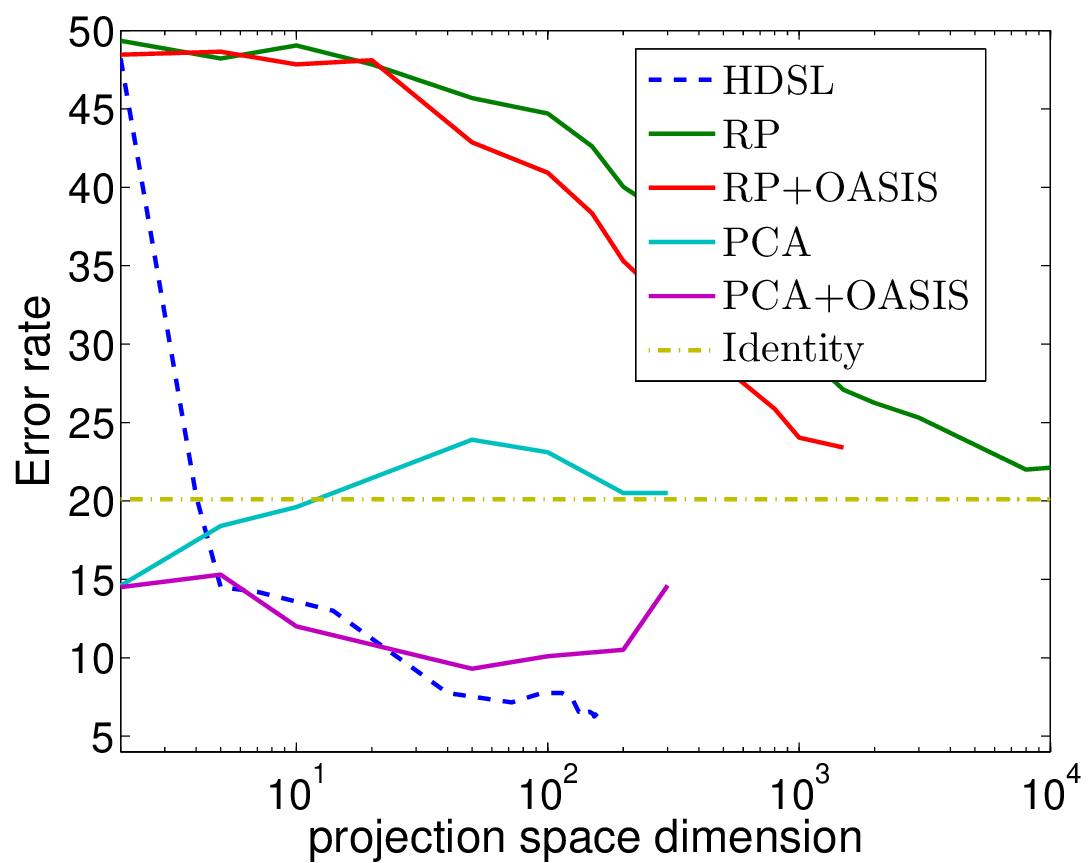}
\label{fErr_dim1}
}
\subfigure[dorothea dataset]{\includegraphics[width=0.41\textwidth]{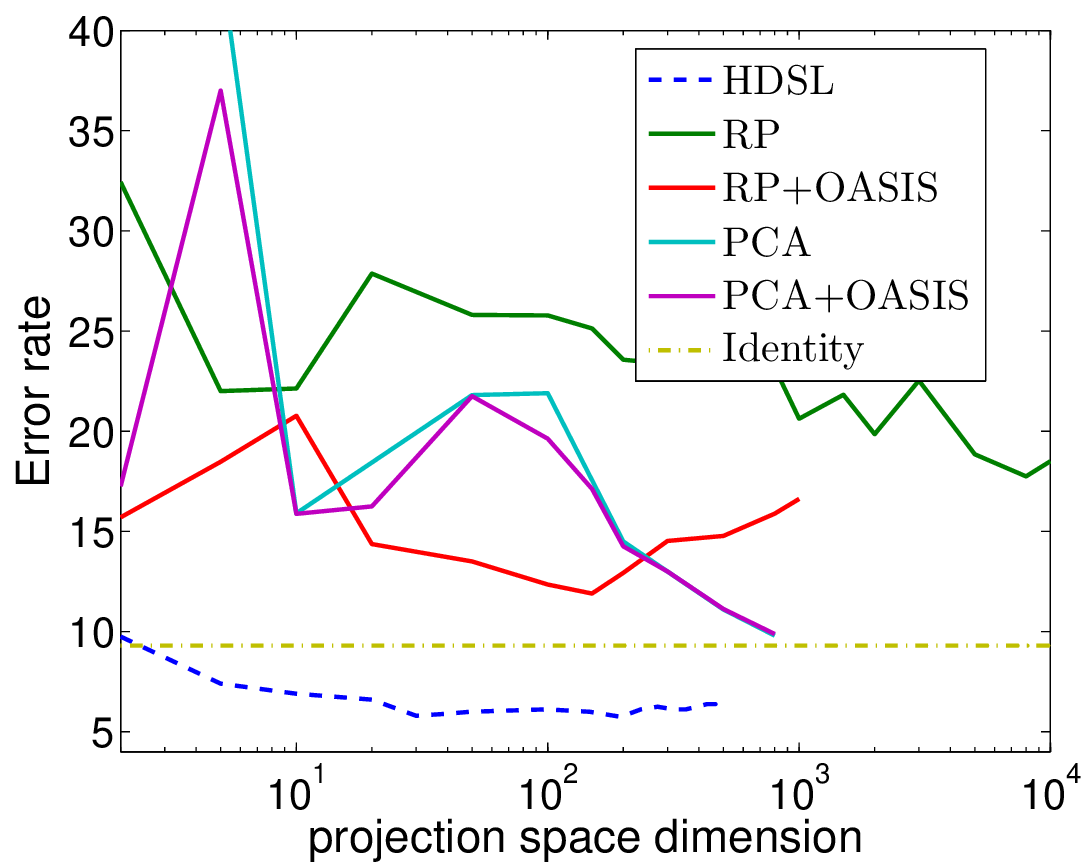} 
\label{fErr_dim4}}
\subfigure[rcv1\_2 dataset]{\includegraphics[width=0.41\textwidth]{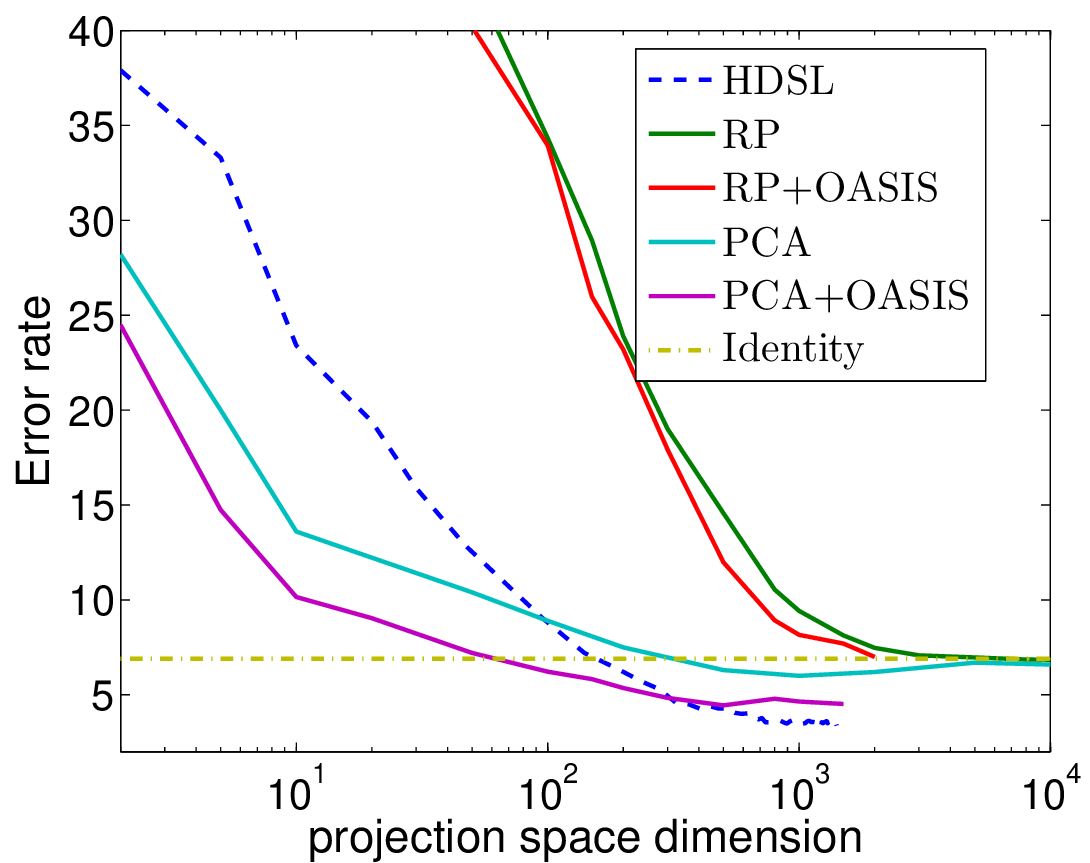} 
\label{fErr_dim2}}
\subfigure[rcv1\_4 dataset]{
\includegraphics[width=0.41\textwidth]{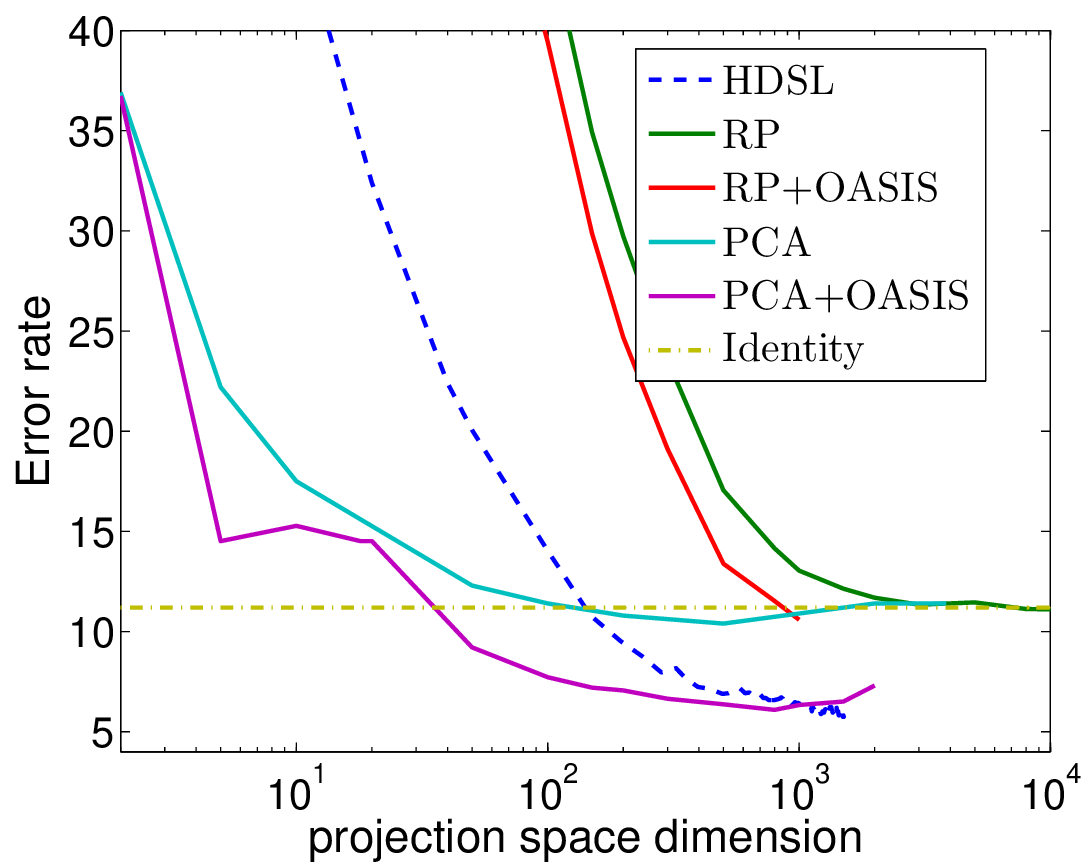}
\label{fErr_dim3}
}
 \caption{$k$-NN test error as a function of the dimensionality of the space (in log scale). Best seen in color.}\label{fErr_dim}
\end{figure}

We now investigate the potential of \hdsl for dimensionality reduction. Recall that \hdsl learns a sequence of PSD matrices $\mat{M}^{(k)}$. We can use the square root of $\mat{M}^{(k)}$ to project the data into a new space where the dot product is equivalent to $S_{\mat{M}^{(k)}}$ in the original space. The dimension of the projection space is equal to the rank of $\mat{M}^{(k)}$, which is upper bounded by $k+1$. A single run of \hdsl can thus be seen as incrementally building projection spaces of increasing dimensionality.

To assess the dimensionality reduction quality of \hdsl (measured by $k$-NN classification error on the test set), we plot its performance at various iterations during the runs that gave the results in Table~\ref{tKnnErr}. We compare it to two standard dimensionality reduction techniques: random projection and PCA. We also evaluate \rpml and \pcaml, i.e., learn a similarity with OASIS on top of the RP and PCA features.\footnote{Again, we were not able to run OASIS beyond a certain dimension due to computational complexity.} Note that OASIS was tuned separately for each projection size, making the comparison a bit unfair to \hdsl. The results are shown in Figure~\ref{fErr_dim}.
As observed earlier, random projection-based approaches achieve poor performance. When the features are not too noisy (as in rcv1\_2 and rcv1\_4), PCA-based methods are better than \hdsl at compressing the space into very few dimensions, but \hdsl eventually catches up. On the other hand, PCA suffers heavily from the presence of noise (dexter and dorothea), while \hdsl is able to quickly improve upon the standard similarity in the original space. Finally, on all datasets, we observe that \hdsl converges to a stationary dimension without overfitting, unlike \pcaml which exhibits signs of overfitting on dexter and rcv1\_4 especially.

\section{Conclusion}
\label{sConclude}

In this work, we proposed an efficient approach to learn similarity functions from high-dimensional sparse data. This is achieved by forming the similarity as a combination of simple sparse basis elements that operate on only two features and the use of an (approximate) Frank-Wolfe algorithm. Experiments on real-world datasets confirmed the robustness of the approach to noisy features and its usefulness for classification and dimensionality reduction. Together with the extreme sparsity of the learned similarity, this makes our approach potentially useful in a variety of other contexts, from data exploration to clustering and ranking, and more generally as a way to preprocess the data before applying any learning algorithm.

\paragraph{Acknowledgments} 
This work was in part supported by the Intelligence Advanced
Research Projects Activity (IARPA) via Department of Defense
U.S. Army Research Laboratory (DoD / ARL) contract
number W911NF-12-C-0012, a NSF IIS-1065243, an Alfred. P. Sloan Research Fellowship, DARPA award D11AP00278, and an ARO YIP Award (W911NF-12-1-0241).  The U.S. Government is authorized
to reproduce and distribute reprints for Governmental
purposes notwithstanding any copyright annotation thereon.
The views and conclusions contained herein are
those of the authors and should not be interpreted as necessarily
representing the official policies or endorsements, either
expressed or implied, of IARPA, DoD/ARL, or the U.S.
Government.

\bibliographystyle{plainnat}
\bibliography{main}

\end{document}